\documentclass[twoside]{article}

\usepackage[accepted]{aistats2024}
\usepackage[round]{natbib}

\usepackage{amsmath,amsthm,amssymb}
\usepackage{mathtools}
\usepackage{bm}
\usepackage{multirow}
\usepackage{cleveref}
\usepackage{listings}
\usepackage{color}
\usepackage{subcaption}
\usepackage{algorithm}
\usepackage{algorithmic}
\usepackage{booktabs}
\usepackage{newfloat}

\definecolor{dkgreen}{rgb}{0,0.6,0}
\definecolor{gray}{rgb}{0.5,0.5,0.5}
\definecolor{mauve}{rgb}{0.58,0,0.82}

\DeclareCaptionStyle{ruled}{labelfont=normalfont,labelsep=colon,strut=off} 
\floatstyle{ruled}
\newfloat{listing}{tb}{lst}{}
\floatname{listing}{Listing}

\newcommand{\NN}{\mathbb{N}}
\newcommand{\RR}{\mathbb{R}}

\newcommand{\mSigma}{\bm{\Sigma}}

\newcommand{\0}{\bm{0}}
\newcommand{\1}{\bm{1}}
\newcommand{\va}{\bm{a}}

\newcommand{\ve}{\bm{e}}

\newcommand{\vu}{\bm{u}}
\newcommand{\vv}{\bm{v}}
\newcommand{\vx}{\bm{x}}
\newcommand{\vy}{\bm{y}}
\newcommand{\vw}{\bm{w}}
\newcommand{\vell}{\bm{\ell}}
\newcommand{\vmu}{\bm{\mu}}
\newcommand{\vsigma}{\bm{\sigma}}

\newcommand{\tr}{\intercal}

\newcommand{\sC}{\mathcal{C}}
\newcommand{\sD}{\mathcal{D}}
\newcommand{\sE}{\mathcal{E}}

\newcommand{\sN}{\mathcal{N}}
\newcommand{\sO}{\mathcal{O}}
\newcommand{\sP}{\mathcal{P}}

\newcommand{\sS}{\mathcal{S}}

\newcommand{\sX}{\mathcal{X}}

\def\rvw{{\mathbf{w}}}

\DeclareMathOperator{\ReLU}{ReLU}

\DeclareMathOperator{\diag}{diag}
\DeclareMathOperator\erf{erf}

\DeclarePairedDelimiter\norm{\lVert}{\rVert}
\DeclarePairedDelimiter\tup{\langle}{\rangle}
\DeclarePairedDelimiter\set{\{}{\}}

\def\safe{{\textnormal{safe}}}
\newcommand{\hatsC}{\hat{\mathcal{C}}}
\newcommand{\Psafe}{P_{\safe}}

\newcommand{\NPcomplete}{NP-complete}

\theoremstyle{definition}
\newtheorem{definition}{Definition}[section]

\theoremstyle{definition}
\newtheorem{example}{Example}[section]

\newtheorem{proposition}{Proposition}[section]

\begin{document}

%
\runningtitle{Tight Verification of Bayesian Neural Networks}

%

\twocolumn[

\aistatstitle{Tight Verification of Probabilistic Robustness\\in Bayesian Neural Networks}

\aistatsauthor{ Ben Batten$^*$ \And Mehran Hosseini$^*$ \And  Alessio Lomuscio }

\runningauthor{ Ben Batten, Mehran Hosseini, Alessio Lomuscio }

\aistatsaddress{ Imperial College London \And  King's College London \And Imperial College London } ]

\begin{abstract}
  We introduce two algorithms for computing tight guarantees on the
  probabilistic robustness of Bayesian Neural Networks
  (BNNs). Computing robustness guarantees for BNNs is a significantly
  more challenging task than verifying the robustness of standard
  Neural Networks (NNs) because it requires searching the parameters'
  space for safe weights. Moreover, tight and complete approaches for
  the verification of standard NNs, such as those based on
  Mixed-Integer Linear Programming (MILP), cannot be directly used for
  the verification of BNNs because of the polynomial terms resulting
  from the consecutive multiplication of variables encoding the
  weights. Our algorithms efficiently and effectively search the
  parameters' space for safe weights by using iterative expansion and
  the network's gradient and can be used with any verification
  algorithm of choice for BNNs. In addition to proving that our
  algorithms compute tighter bounds than the SoA, we also evaluate our
  algorithms against the SoA on standard benchmarks, such as MNIST and
  CIFAR10, showing that our algorithms compute bounds up to 40\%
  tighter than the SoA.$^{\dagger}$
\end{abstract}

\section{Introduction}
\label{sec: Introduction}
\vspace{-1mm} In recent years \emph{Neural Networks} (\emph{NNs}) have
been proposed to perform safety-critical functions in applications
such as automated driving, medical image processing and
beyond~\citep{Abiodun+18, Anwar+18, Bojarski+16, Liu21}. For machine
learning methods to be adopted in these areas, assurance and
certification are paramount.

Following the known vulnerabilities of NNs to adversarial attacks
\citep{Dalvi+04, Szegedy+13, Ilyas+19, Sharif+16, GoodfellowSS14}, a
rapidly growing body of research has focused on methods to verify the
safety of NNs~\citep{PulinaT10, LomuscioM17, Katz+17, MadryMSTV18,
  Liu+21, Zhang+23} to identify fragilities before deployment.

\let\thefootnote\relax\footnotetext{\hspace{-1.8em}$^{*}$Equal contribution.}
\let\thefootnote\relax\footnotetext{\hspace{-1.8em}$^{\dagger}$The code is available at \texttt{https://github.com/benbatten/\\TightVerificationOfProbabilisticRobustnessInBNNs}.}
\def\thefootnote{\arabic{footnote}}

To mitigate NN vulnerabilities to adversarial attacks, several methods
have been proposed, including adversarial training \citep{MadryMSTV18,
  Tramer+18, GaninUAGLLML17, Shafahi+19}, and verification-based
augmentation~\citep{Dreossi+19}. While considerable progress has been
made in these areas, present models remain fragile to various input
perturbations, thereby limiting their potential applicability in key
areas.

In an attempt to overcome this problem, architectures that are
inherently more robust have been proposed. Bayesian Neural Networks
(BNNs) have arguably emerged as a key class of models intrinsically
more resilient to adversarial attacks~\citep{Pang+21,
  Carbone+20}. This is because BNNs naturally incorporate the
randomness present in the posterior distribution \citep{UchenduCMH21}.

Although BNNs have been shown to provide a degree of natural
robustness, it is crucial to verify their safety and correctness with
respect to specifications of interest, including robustness, before
deployment. Note that while safety is formulated as a standard
decision problem in deterministic NNs, for BNNs, safety is a
probabilistic concept. Specifically, given a BNN, an input region and
a safe set, the probabilistic robustness problem concerns determining
the probability that the input neighbourhood is mapped to a safe set
by the BNN~\citep{Cardelli+19, CardelliKLP19, WickerLPK20}.

The probabilistic robustness problem for BNNs is also sometimes
referred to as the \emph{probabilistic safety problem}
\citep{CardelliKLP19}. Since the exact computation of probabilistic
robustness for BNNs is computationally intractable, several approaches
for approximating it have been proposed~\citep{WickerLPK20,
  Berrada+21, Lechner+21}.

A key challenge in this area is to approximate the probabilistic
robustness of BNNs tightly and efficiently. As we demonstrate in
Sections~\ref{sec: Verification} and \ref{sec: Evaluation}, existing
methods are either loose or computationally expensive. It therefore
remains of great interest to explore new approaches that can derive
tighter approximations to probabilistic robustness while remaining
computationally effective. To this end, our contributions in this
paper can be summarised as follows.
\begin{itemize}
\item We introduce two new algorithms, \emph{Pure Iterative Expansion}
  (\emph{PIE}) and \emph{Gradient-guided Iterative Expansion}
  (\emph{GIE}). We show that they produce sound lower bounds on the
  probabilistic robustness of BNNs, which are provably tighter than
  SoA sampling-based approaches.
\item We empirically evaluate our algorithms on MNIST and CIFAR10
  benchmarks against the SoA and show that our approach produces
  tighter approximations of the probabilistic robustness.
\item For reproducibility, we provide our code in the supplementary
  material.
\end{itemize}

The rest of the paper is organised as follows. In
Subsection~\ref{subsec: Related Work}, we review the related
literature on NN and BNN verification. In Section~\ref{sec:
  Preliminaries}, we formally define BNNs, the probabilistic
robustness problem for BNNs, and other related terminology that we use
throughout the paper. We discuss the details of our approach in
Section~\ref{sec: Verification} by introducing the PIE and GIE
algorithms for exploring the set of safe weights and computing lower
bounds for the probabilistic robustness of BNNs in
Subsection~\ref{subsec: Gradient-Guided Tiling}. We also formally show
that our approach outperforms the SoA sampling-based method in
Subsection~\ref{subsec: Comparison} by producing provably tighter
bounds. We evaluate the performance of our approach against SoA on the
MNIST and CIFAR10 datasets in Section~\ref{sec: Evaluation}.

\subsection*{Related Work}
\label{subsec: Related Work}
\emph{Mixed-Integer Linear Programming} (\emph{MILP}) and \emph{Bound
  Propagation} (\emph{BP}) are two common techniques used for the
verification of standard NNs. In MILP-based approaches, the
verification problem for NNs is encoded as an MILP problem
\citep{LomuscioM17, HosseiniL23, TjengXT19}, which can be solved using
MILP solvers. A common assumption in MILP-based approaches is that
activations are piecewise-linear.

Propagation-based approaches differ from MILP-based approaches in that
they consider a relaxed network and propagate bounds on each node's
output through the network's layers to solve the verification
problem~\citep{Gowal+18, SinghGPV19, SinghGMPV18, Wang+19}.

SoA approaches for the verification of NNs combine MILP and BP to
achieve complete and scalable verification for NNs. For instance,
\citep{KouvarosL21} uses bound propagation alongside \emph{adaptive
  splitting} \citep{HenriksenL20, BotoevaKKLM20} to reduce the number
of binary variables in the resulting MILP problem and thus further
speed up the verification process.
An important and well-studied safety specification in NN verification
is \emph{robustness}, i.e., a model's resilience in returning the same
output for all values in a small neighbourhood of a given input. An
example approach devoted only to checking robustness
is~\citep{CardelliKLP19}, where they compute upper bounds on the
robustness of BNNs.
Closely related to this, \citep{Cardelli+19} considers the problem of
computing the probability of whether, for a given input to a BNN,
there exists another input within a bounded set around the input such
that the network's prediction differs for the two points and provides
an approach to approximate this probability and compute statistical
guarantees for it.

In the same line of research \citep{Carbone+20} shows that the
vulnerability of BNNs to gradient-based attacks arises as a result of
the data lying on a lower-dimensional submanifold of the ambient
space.

\citep{WickerLPK20} considers a set-up similar to this paper, using a
sampling-based approach to bound the probabilistic robustness of
BNNs. In this approach, they sample orthotopes from the posterior
distribution and use \emph{Interval Bound Propagation} (\emph{IBP})
and \emph{Linear Bound Propagation} (\emph{LBP}) to determine whether
a given orthotope is safe. By repeating this process, they obtain a
robustness certificate and a lower bound for the probabilistic
robustness of the BNN.

\citet{Berrada+21} consider the \emph{adversarial robustness} of BNNs.
They use the notion of Lagrangian duality, which replaces standard
Lagrangian multipliers with functional multipliers. For an optimal
choice of functional multipliers, this approach leads to exact
verification, i.e., the exact computation of probabilistic
robustness. Nevertheless, this may not be practical. The authors show
that for specific classes of multipliers, however, one can obtain
lower bounds for the probabilistic robustness.

\citet{AdamsPLL23} introduced an algorithm to verify the adversarial
robustness of BNNs. They use Dynamic Programming alongside bound
propagation and convex relaxations to compute lower and upper bounds
on BNN predictions for a given input set. \citet{WickerPLW23} build on
the algorithm first presented in \citep{WickerLPK20} to provide upper
bounds on the probabilistic robustness of BNNs in addition to lower
bounds. Moreover, they extend the algorithmic formulation in
\citep{WickerLPK20} to handle the specification proposed in
\citep{Berrada+21} on posterior distributions with unbounded support.

\citep{Lechner+21} uses MILP to obtain robustness guarantees for
infinite time-horizon safety and robustness of BNNs. However, this
approach can only be used for small networks. \citep{WijkWK22} obtains
robustness guarantees for credal BNNs via constraint relaxation over
probabilistic circuits.


%
\section{Preliminaries}
\label{sec: Preliminaries}
Throughout this paper, we denote the set of real numbers by \(\RR\),
use italics, \(w\), to refer to scalers in \(\RR\), bold italics,
\(\vw\), for vectors in \(\RR^n\), and bold letters, \(\rvw\), to
refer to random variables in \(\RR^n\). For an orthotope
\(\sO \subset \RR^n\) with two opposing corners \(\vw_1, \vw_2\), we
denote \(\sO\) by \([\vw_1, \vw_2]\).

In the rest of this section, we define BNNs and briefly discuss common
approaches for training them in Subsection~\ref{subsec: BNNs}; we then
define the robustness problem for BNNs and a few related concepts,
which we will use in Section~\ref{sec: Verification}, in
Subsection~\ref{subsec: Probabilistic Safety}.

\subsection{Bayesian Neural Networks}
\label{subsec: BNNs}
\begin{definition}[Bayesian Neural Network]
  \label{def: BNN}
  Given a dataset, \(\sD\), let \(f_{\rvw}: \RR^m \to \RR^n\) be a
  \emph{Bayesian neural network} with \(\rvw\) a vector random
  variable consisting of the parameters of the network. Assuming a
  prior distribution over the parameters, \(\rvw \sim p(\vw)\),
  training amounts to computing the posterior distribution,
  \(p(\vw | \sD)\), by the application of Baye's rule. For a vector,
  $\vw$, sampled from the posterior, \(p(\vw | \sD)\), we denote the
  deterministic neural network with parameters \(\vw\) by \(f_{\vw}\).
  We use \(n_{\vw}\) to refer to the dimension of the parameters space
  of the network.
\end{definition}

In BNN literature, it is customary to assume the prior and the
posterior both have normal distribution; nevertheless, our
\emph{Pure Iterative Expansion} (\emph{PIE}) algorithm works for
arbitrary distributions, and our \emph{Gradient-guided Iterative
  Expansion} (\emph{GIE}) algorithm works for all continuous
distributions. In the case of normal distribution, training requires
finding the mean \(\vmu \in \RR^{n_{\vw}}\) and covariance
\(\mSigma \in \RR^{n_{\vw} \times n_{\vw}}\) of the parameters such
that \(p(\vw | \sD) \approx q(\vw) \sim \sN(\vmu, \mSigma)\). We refer
to the parameters' variance by \(\vsigma^2 = \diag(\mSigma)\). When
the parameters are independent, i.e., when the covariance matrix
\(\mSigma\) is diagonal, we denote a normal distribution with mean
\(\vmu\) and covariance matrix \(\mSigma\) by
\(\sN(\vmu, \vsigma^2)\). Similarly, we denote an arbitrary
probability distribution with mean \(\vmu\) and variance \(\vsigma^2\)
by \(\sP(\vmu, \vsigma^2)\).

Despite the similarities between standard, deterministic NNs and BNNs,
training BNNs poses a greater challenge since computing the posterior
of a BNN is computationally infeasible due to non-linearities in the
network~ \citep{Mackay92}. Some of the approaches for calculating the
posterior distribution include \emph{Variational Inference}
(\emph{VI}) \citep{Blundell+15}, \emph{Markov Chain Monte Carlo}
(\emph{MCMC}) \citep{Hastings70}, combinations of VI and MCMC
\citep{Salimans+15}, or other approximate approaches. For a comparison
of the approaches see \citep{Jospin+22}.

\subsection{The Probabilistic Robustness Problem for BNNs}
\label{subsec: Probabilistic Safety}
Let us start by defining the probabilistic robustness problem for
BNNs.
\begin{definition}[Probabilistic Robustness Problem]
  \label{def: Probabilistic Safety and Accuracy}
  Given a BNN \(f_{\rvw}: \RR^m \to \RR^n\), a \emph{robustness
    specification} for \(f_{\rvw}\) is a tuple
  \(\varphi = \tup{\sX, \sS}\), where \(\sX \subseteq \RR^m \) is an
  \(\epsilon\)-ball \(B_{\epsilon}(\vx)\) around a given
  \(\vx \in \RR^m\) with respect to a norm \(\norm{\cdot}_p\) and
  \(\sS \subseteq \RR^n\) is a half-space
  \begin{equation}
    \label{eq: Half Space}
    \sS = \set{\vy : \va^{\tr} \cdot \vy \geq b},
  \end{equation}
  for some \(\va \in \RR^n\) and \(b \in \RR\). The
  \emph{probabilistic robustness problem} concerns determining the
  probability that for all inputs in \(\sX\) the output of \(f_\rvw\) is in
  \(\sS\), i.e.,
  \begin{equation*}
    \label{eq: Probabilistic Safety}
    P_{\safe}(\sX, \sS) = P_{\vw \sim \rvw}(\forall \vx \in \sX, f_{\vw}(\vx) \in \sS).
  \end{equation*}
\end{definition}

We refer to \(\sX\) as the \emph{input constraint} set and \(\sS\) as
the \emph{output constraint} set. The probabilistic robustness problem
is sometimes referred to as the \emph{probabilistic safety problem}
\citep{CardelliKLP19} or the \emph{adversarial robustness problem}
\citep{Berrada+21} for BNNs; in the same light, $\varphi$ is also
referred to as the \emph{safety specification}.

For a classification task, when the goal is to verify for a correct
class \(c\) against a target class \(c'\), we fix $\va_c = 1$,
$\va_{c'} = -1$, and $\va_i = 0$ for, $i \neq c, c'$ in
Equation~\eqref{eq: Half Space}.

\begin{definition}[Set of Safe Weights]
  The \emph{set of safe weights} for a BNN with respect to a
  specification \(\varphi = \tup{\sX, \sS}\), is defined as
  \begin{equation}
    \label{eq: Set of Safe Weights}
    \sC = \set{\vw \in \RR^{n_{\vw}} : \forall \vx \in \sX, f_{\vw}(\vx) \in \sS}.
  \end{equation}
\end{definition}%

\begin{proposition}
  \label{prop: Integral}
  Given a BNN \(f_{\rvw}\), such that \(\rvw \sim q(\vw)\) for a
  probability density function \(q\), and a safety specification
  \(\tup{\sX, \sS}\), we have that
  \begin{equation*}
    \label{eq: Safe Weights to Probabilistic Safety}
    P_\safe (\sX, \sS) = \int_\sC q(\vw) d\vw.
  \end{equation*}
\end{proposition}

Since solving the robustness problem for deterministic networks is
\NPcomplete{} \citep{Katz+17, Sinha+18}, verifying whether a given
\(\vw \in \sC\) is also \NPcomplete. Therefore, it is infeasible to
calculate \(\Psafe\) exactly even for moderately sized networks.
In the rest of this paper we put forward two algorithms for
calculating lower-bound approximations of \(\Psafe\) in a provably
tighter way than previous approaches.


%
\section{Computing the Probabilistic Robustness for BNNs}
\label{sec: Verification}
In Subsection~\ref{subsec: Probabilistic Safety} we discussed why
calculating \(\Psafe\) is computationally infeasible. Besides
empirical sampling, the current approaches for computing \(\Psafe\)
rely on sampling orthotopes \(\hatsC_i \subseteq \RR^{n_{\vw}}\) with
fixed sizes, whose edges are proportional to the standard deviations
of the weights of the network, and verifying whether
\(f_{\hatsC_i}(\sX) \subseteq \sS\) through \emph{Linear Bound
  Propagation} (LBP) \citep{WickerLPK20}. This enables the computation
of an under-approximation \(\hatsC = \bigcup \hatsC_i \subseteq \sC\),
which, in turn, allows computing the lower bound
\begin{equation*}
  \label{eq: Lower Bound}
  p_{\text{safe}} = \int_{\hatsC} q(\vw) d\vw \leq P_{\safe}(\sX, \sS).
\end{equation*}
%

One weakness of this approach is that it uses orthotopes of the same
size for all specifications. For example, whilst for a given input to
a BNN it may be possible to fit large orthotopes around the weights's
mean and use LBP to verify the safety of these orthotopes, for another
input to the same network this may not necessarily be the case, and
must resort to smaller orthotopes, which in high dimensions result in
obtaining minuscule probabilistic robustness lower bounds, a
phenomenon demonstrated in Example~\ref{ex: Bounded Support}.

To mitigate this problem, we introduce two approaches that allow us to
dynamically adjust the size of orthotopes
\(\hatsC_1, \dots, \hatsC_s\). The \emph{Pure Iterative Expansion}
(\emph{PIE}) approach, which we introduce in Subsection~\ref{subsec:
  Dynamic Scaling}, allows us to dynamically \emph{scale} the size of
and orthotope \(\hatsC_i\), and therefore, cover a larger volume in
\(\sC\).

The more sophisticated \emph{Gradient-guided Iterative Expansion}
(\emph{GIE}) approach, which we introduce in Subsection~\ref{subsec:
  Gradient-Guided Tiling}, not only allows the use of dynamic scaling,
but also expands the orthotopes further in the dimensions that allow
\(\hatsC_i\)'s to remain in \(\sC\), according to the posterior's
gradient, \(\nabla_{\rvw}f_{\rvw}\).

Both PIE and GIE can be accompanied by LBP or other approaches. Since
our iterative expansion approach depends on iteratively expanding
\(\hatsC_i\)'s and verifying their safety, we use LBP, as described
and implemented in~\citep{WickerLPK20}, because of its significant
computational efficiency.

\subsection{Pure Dynamic Scaling}
\label{subsec: Dynamic Scaling}
Given a BNN, \(f_{\rvw}\), and safety specification,
\(\varphi = \tup{\sX, \sS}\), remember that we denote the set of safe
weights of \(f_{\vw}\) with respect to \(\varphi\) as \(\sC\). The
pure sampling and LBP approach \citep{WickerLPK20} starts by the
under-approximation \(\hatsC = \set{}\) of \(\sC\) and samples
\(\vw_0 \sim q(\vw)\). It then defines
\(\hatsC_0 = [\vw_0\pm \lambda \vsigma]\), where \(\lambda > 0\) is
the scaling factor and \(\vsigma\) is the parameters' standard
deviation vector. Then, using IBP or LBP, it checks whether
\(\hatsC_0 \subseteq \sC\). If this is the case,
\(\hatsC = \hatsC \cup \hatsC_0\); otherwise, \(\hatsC_0\) is
discarded. This process is repeated until a given number of weights
\(s\) have been tried.

As we briefly discussed in the beginning of Section~\ref{sec:
  Verification} and demonstrate in Section~\ref{sec: Evaluation}, the
lower bounds computed using a pure sampling approach are not tight,
because of the rigidity of the sampled orthotopes.

The first and most straightforward solution to this problem that we
propose is to instead use dynamically-scaled orthotopes, as outlined
in Algorithm~\ref{alg: PIE}. Wherein, after sampling an orthotope
\(\hatsC_i = [\vw_i\pm \lambda \vsigma] \subseteq \sC\) and verifying
it (line 8 in Algorithm~\ref{alg: PIE}), we expand the orthotope to
\(\hatsC_i= [\vw_i\pm 2\lambda \vsigma]\) rather than immediately
moving on to the next sampled parameter vector, as in the pure
sampling approach. After the initial expansion, we again check
\(\hatsC_i \subseteq \hatsC\). If this is the case, we can repeat the
same process for \(\hatsC_i= [\vw_i\pm 3 \lambda \vsigma]\). We repeat
this, until \(\hatsC_i \not\subseteq \sC\).

\begin{algorithm}[t]
  \caption{Certifying BNNs using PIE}
  \label{alg: PIE}
  \begin{algorithmic}[1]
    \STATE {\bfseries Input:} BNN \(f_{\rvw}\),
    \(\rvw \sim \sP(\vmu, \boldsymbol{\sigma}^2)\), safety specification
    \(\varphi = \tup{\sX, \sS}\), number of samples \(s \in \NN\),
    and scaling factor \(\lambda > 0\).
    \STATE {\bfseries Output:} \(\hat{\sC}\) and \(p_{\text{safe}}\).
    \STATE \(\hat{\sC} = \set{}\)\;
    \FOR{\(i \gets 1\) to \(s\)}
    \STATE \(\vw \sim \sP(\vmu, \boldsymbol{\sigma}^2)\)
    \STATE \(j = 1\)
    \STATE \(\bar{\sC} = \set{}\)
    \WHILE{\(LBP([\vw \pm j \lambda \vsigma], \sX) \subseteq \sS\)}
    \STATE \(\bar{\sC} = [\vw \pm j \lambda \vsigma]\)
    \STATE \(j = j + 1\)
    \ENDWHILE
    \STATE \(\hatsC = \hatsC \cup \bar{\sC}\)
    \ENDFOR
    \STATE {\bfseries Return} \(\hat{\sC}\) and \(\int_{\hat{\sC}} q(\vw) d\vw\)
  \end{algorithmic}
\end{algorithm}

Example~\ref{ex: Iterative Expansion} demonstrates the advantage of
the PIE approach over the pure sampling-based approach.
\begin{example}
  \label{ex: Iterative Expansion}
  Consider a BNN with two Bayesian layers,
  \(f_{\vw}(x) = \ReLU(w_2 \cdot \ReLU(w_1 \cdot x))\), where
  \(w_, w_2 \sim \sN(0, 1)\) and no bias terms. For the input
  constraint set \(\sX = \{1\}\) and the output constraint set
  \(\sS = \set{y \in \RR : y \leq 1}\), the set of safe weights is
  \(\sC = \left\{\begin{pmatrix}w_1 & w_2\end{pmatrix}^{\tr}\in\RR^2 :
    \ReLU(w_2 \cdot \ReLU(w_1)) \leq 1\right\}.\) The set \(\sC\) is
  shown blue in Figures~\ref{subfig: Pure Sampling} to \ref{subfig:
    Gradient-Guided}.
    
  For the same sampled weights
  \(\vw_1, \dots, \vw_4 \sim \sN(\0, \1)\) and scaling factor
  \(\lambda=1\), areas covered by the pure sampling approach, PIE
  approach for 2 iterations, and GIE approach
  (cf. Subsection~\ref{subsec: Gradient-Guided Tiling}) for 2
  iterations (see Subsection~\ref{subsec: Gradient-Guided Tiling}) are
  shown in Figures~\ref{subfig: Pure Sampling}, \ref{subfig: Iterative
    Expansion}, and \ref{subfig: Gradient-Guided}, respectively. We
  observe that the GIE approach (Figure~\ref{subfig: Gradient-Guided})
  has the greatest coverage of \(\sC\), followed by the the PIE
  approach (Figure~\ref{subfig: Iterative Expansion}), which has
  greater coverage than the pure sampling approach
  (Figure~\ref{subfig: Pure Sampling}).
\end{example}
\begin{figure*}[t]
  \begin{center}
    \begin{subfigure}[b]{0.32\textwidth}
      \centering
      \includegraphics[width=\textwidth]{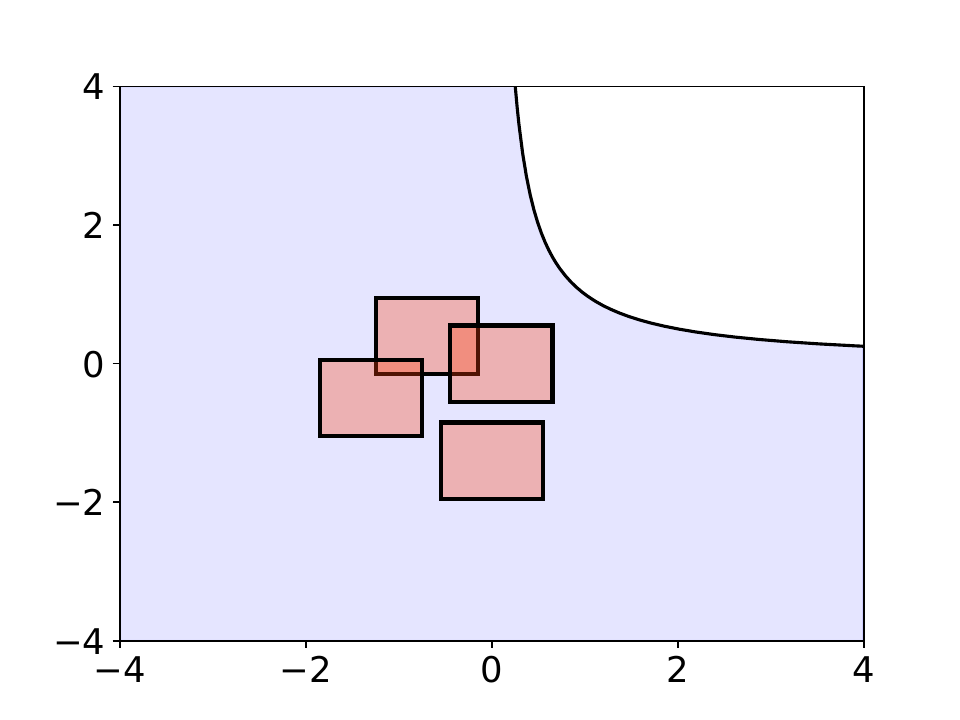}
      \caption{Pure sampling.}
      \label{subfig: Pure Sampling}
    \end{subfigure}
    \hfill
    \begin{subfigure}[b]{0.32\textwidth}
      \centering
      \includegraphics[width=\textwidth]{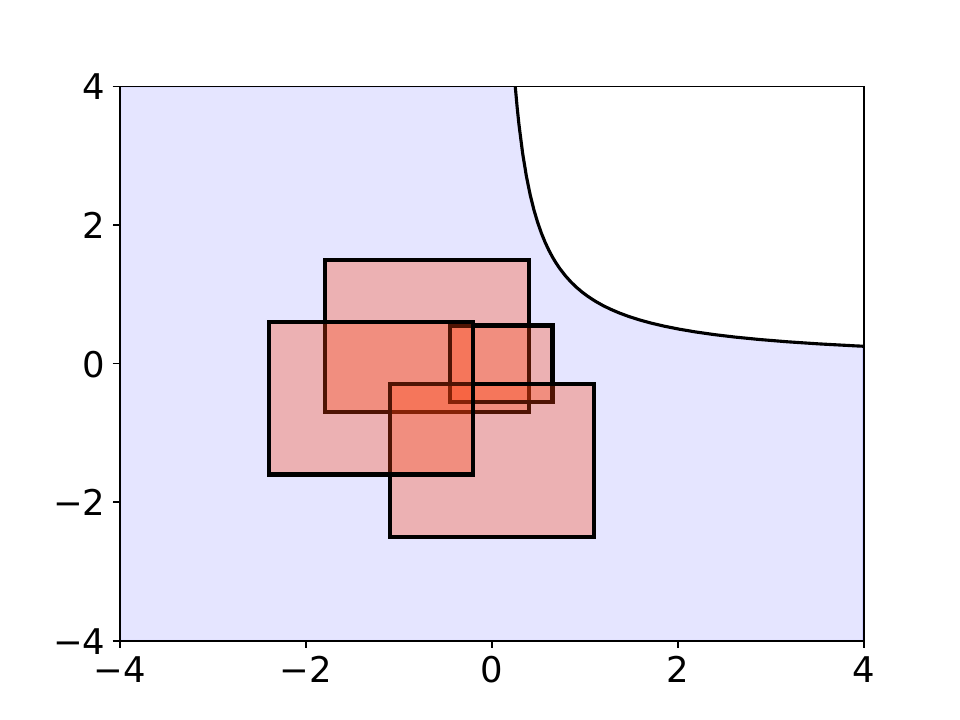}
      \caption{PIE.}
      \label{subfig: Iterative Expansion}
    \end{subfigure}
    \hfill
    \begin{subfigure}[b]{0.32\textwidth}
      \centering
      \includegraphics[width=\textwidth]{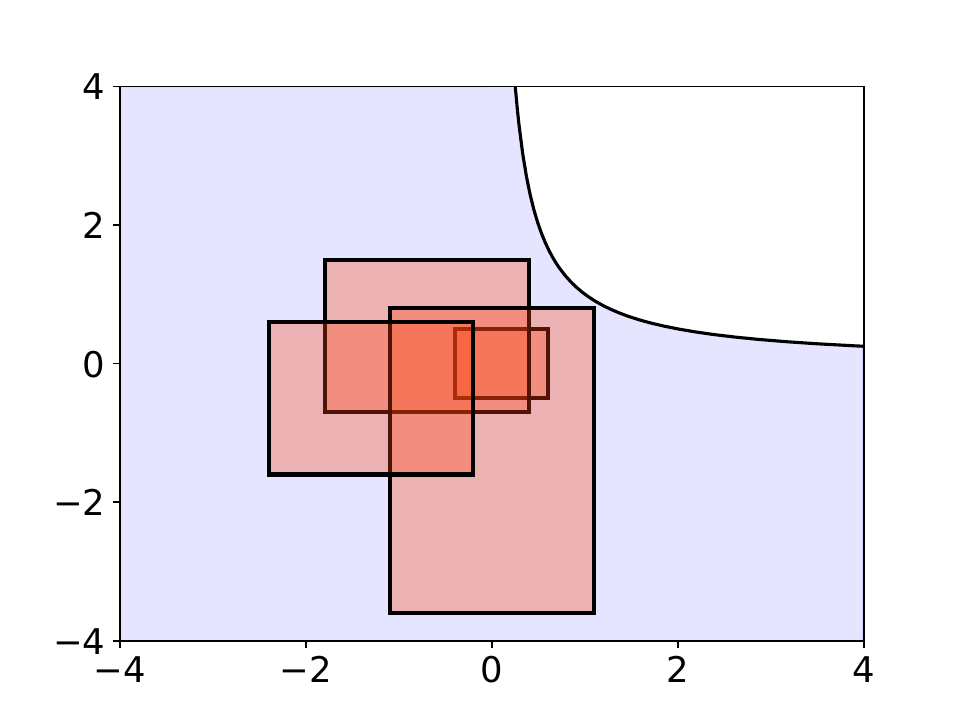}
      \caption{GIE.}
      \label{subfig: Gradient-Guided}
    \end{subfigure}
  \end{center}
  \caption{A comparison between the area of the set of safe weights
    \(\sC\) in Example~\ref{ex: Iterative Expansion} that is covered
    by the pure sampling (\ref{subfig: Pure Sampling}), PIE after two
    iterations (\ref{subfig: Iterative Expansion}), and GIE after two
    iterations (\ref{subfig: Gradient-Guided}) for the same initial
    weights \(\vw_1, \dots, \vw_4 \sim \sN(\0, \1)\). See
    Section~\ref{sec: Evaluation} for detailed comparisons, including
    comparison with equal compute.}
  \label{fig: 2D Example}
\end{figure*}
\subsection{Gradient-Guided Dynamic Scaling}
\label{subsec: Gradient-Guided Tiling}
The goal of this subsection is to improve Algorithm~\ref{alg: PIE}
even further by using the gradient of the network,
\(\nabla_{\rvw}f_{\rvw}\), to maximise our coverage of the safe set of
weights, $\mathcal{C}$, as summarised in Algorithm~\ref{alg:
  GIE}. Recall that for a given point in the parameter space,
\(\vw \in \sC\), our goal is to explore \(\sC\) starting from
\(\vw\). Instead of expanding \(\hatsC_i\) uniformly in all
dimensions, we can use the parameter-wise BNN gradient,
\(\nabla_{\rvw}f_{\rvw}\), to expand \(\hatsC_i\) further in
dimensions that do not violate \(\hatsC_i \subseteq \sC\).

From Equation~\eqref{eq: Set of Safe Weights}, we have that
\begin{equation*}
  \sC = \set{\vw \in \RR^{n_{\vw}} :
    \va^{\tr} \cdot f_{\vw}(B_{\epsilon}(\vx)) \geq b},
\end{equation*}
for some \(\va \in \RR^n\), \(\vx \in \RR^m\), \(b \in \RR\), and
\(\epsilon > 0\). To maximise the volume of a candidate orthotope,
\(\hatsC_i\), while staying within the bounds of \(\sC\), we propose
expanding further in directions with \emph{zero} or \emph{positive
  gradients} in $\nabla_{\vw}\va^{\tr} \cdot
f_{\rvw}(\vx)$. Concretely, let
\begin{equation*}
  \begin{split}
    \sE^{-} & = \set{\ve_i \in \RR^{n_{\vw}} : \nabla_{\rvw} (\va^{\tr} \cdot f_{\vw})(\vx) \cdot \ve_i < 0}, \\
    \sE^{+} & = \set{\ve_i \in \RR^{n_{\vw}} : \nabla_{\rvw} (\va^{\tr} \cdot f_{\vw})(\vx) \cdot \ve_i > 0}, \\
    \sE^{0} \  & = \set{\ve_i \in \RR^{n_{\vw}} : \nabla_{\rvw} (\va^{\tr} \cdot f_{\vw})(\vx) \cdot \ve_i = 0},
  \end{split}
\end{equation*}
where \(\set{\ve_1, \dots, \ve_{n_{\vw}}}\) is the standard basis for
the \(\RR^{n_{\vw}}\).

Intuitively, we can cover more of $\sC$ by moving further in
directions that take us away from the boundary at
$\va^{\tr} \cdot f_{\vw}(\vx) = b$. In particular, when the
activations considered are \(\ReLU\), where we have that
\(\nabla_{\rvw} (\va^{\tr} \cdot f_{\vw_i})(\vx) \cdot \ve_i = 0\) for
all \(i\)'s corresponding to a weight feeding directly to an inactive
node. Moreover, the value of a node is a locally linear function of
the weights that are input to that node. Therefore, we have the
freedom of moving in the directions corresponding to these weights
without changing the output of the network, and hence, remaining
inside \(\sC\). We note that this assumption only holds locally, and
with larger deviations in the parameter space, the activation pattern
of the network will change. However, we observe that assuming local
linearity yields a substantial increase in the volume of verifiable
\(\hatsC_i\)'s.

With this knowledge, let us define
\begin{equation*}
  \ve^+ = \!\!\!\sum_{\ve_i \in \sE^+ \cup \sE^0} \!\!\ve_i, \quad\text{and}\quad
  \ve^- = \!\!\!\sum_{\ve_i \in \sE^- \cup \sE^0} \!\!\ve_i.
\end{equation*}
\begin{algorithm}[t]
  \caption{Certifying BNNs using GIE}
  \label{alg: GIE}
  \begin{algorithmic}[1]
    \STATE {\bfseries Input:} BNN \(f_{\rvw}\),
    \(\rvw \sim \sP(\vmu, \boldsymbol{\sigma}^2)\), safety specification
    \(\varphi = \tup{\sX, \sS}\), number of samples \(s \in \NN\),
    scaling factor \(\lambda > 0\), and gradient-based scaling
    factor \(\rho > 0\).
    \STATE {\bfseries Output:} \(\hat{\sC}\) and \(p_{\text{safe}}\).
    \STATE \(\hat{\sC} = \set{}\)\;
    \FOR{\(i \gets 1\) to \(s\)}
    \STATE \(\vw \sim \sP(\vmu, \boldsymbol{\sigma}^2)\)
    \STATE \(\sE^{-} = \set{\ve_i \in \RR^{n_{\vw}} : \nabla_{\rvw} (\va^{\tr} \cdot f_{\vw})(\vx) \cdot \ve_i < 0}\)
    \STATE \(\sE^{+} = \set{\ve_i \in \RR^{n_{\vw}} : \nabla_{\rvw} (\va^{\tr} \cdot f_{\vw})(\vx) \cdot \ve_i > 0}\)
    \STATE \(\sE^{0} \  = \set{\ve_i \in \RR^{n_{\vw}} : \nabla_{\rvw} (\va^{\tr} \cdot f_{\vw})(\vx) \cdot \ve_i = 0}\)
    \STATE \(\displaystyle \vv^+ = \lambda\vsigma(\1 + \rho\!\!\!\sum_{\ve_i \in \sE^+ \cup \sE^0} \!\!\ve_i)\)
    \STATE \(\displaystyle \vv^- = \lambda\vsigma(\1 + \rho\!\!\!\sum_{\ve_i \in \sE^- \cup \sE^0} \!\!\ve_i)\)
    \STATE \(j = 1\)
    \STATE \(\bar{\sC} = \set{}\)
    \WHILE{\(LBP([\vw_i - j\vv^-, \vw_i + j\vv^+], \sX) \subseteq \sS\)}
    \STATE \(\bar{\sC} = [\vw_i - j\vv^-, \vw_i + j\vv^+]\)
    \STATE \(j = j + 1\)
    \ENDWHILE
    \STATE \(\hatsC = \hatsC \cup \bar{\sC}\)
    \ENDFOR
    \STATE {\bfseries Return} \(\hat{\sC}\) and \(\int_{\hat{\sC}} q(\vw) d\vw\)
  \end{algorithmic}
\end{algorithm}

Now, instead of starting from the orthotope
\(\hatsC_i = [\vw_i \pm \lambda \vsigma]\) as in the PIE approach,
we start from
\(\hatsC_i = [\vw_i - \lambda \vsigma (\1 + \rho \ve^-), \vw_i +
\lambda \vsigma (\1 + \rho \ve^+)]\), where \(\rho > 0\) is the
\emph{gradient-based scaling factor}. Then, we check whether
\(\hatsC_i \subseteq \hatsC\) (line 12 in Algorithm~\ref{alg: GIE})
and iteratively expand \(\hatsC_i\), similarly to the PIE approach,
until \(\hatsC_i \not\subseteq \sC\).

\subsection{Theoretical Comparison}
\label{subsec: Comparison}
Here we compare Algorithms~\ref{alg: PIE} and \ref{alg: GIE} against
the SoA sampling-based approach~\citep{WickerLPK20} and show that both
algorithms introduced here produce provably tighter lower bounds for
the probabilistic robustness of BNNs. We start this by comparing
Algorithms~\ref{alg: PIE} and \ref{alg: GIE} against the static
orthotopes approach of \citet{WickerLPK20}. We show that
Algorithms~\ref{alg: PIE} and \ref{alg: GIE} always provides tighter
lower bounds than the static orthotopes approach. We provide an
intuition for the three approaches in Figure~\ref{fig: CIFAR Covered
  Area}.
\begin{proposition}
  For every BNN, \(f_{\rvw}\), and specification,
  \(\phi = \tup{\sX, \sS}\), let \(p\) be the probabilistic robustness
  of \(f_{\rvw}\) with respect to \(\phi\), and \(p_s, p_p\), and
  \(p_g\) be the probabilistic robustness lower bounds computed with
  the same sampled weights \(\vw_1, \dots, \vw_s\) using the pure
  sampling, PIE, and GIE approaches, respectively. Then we have that
  \(p_s \leq p_p, p_g \leq p\).
\end{proposition}
\begin{proof}
  It is trivial that \(p_s, p_p, p_g \leq p\). We show that
  \(p_s \leq p_p\). Let \(\hatsC_s\) and \(\hatsC_v\) be the
  robustness certificates computed using \(\vw_1, \dots, \vw_s\) by
  the pure sampling approach \citep{WickerLPK20} and the PIE approach
  (Algorithm~\ref{alg: PIE}), respectively. For each \(\vw_i\), the
  pure sampling approach checks whether
  \(LBP([\vw_i \pm \lambda \vsigma], \sX) \subseteq \sS\). This is
  also checked in line \(8\) of Algorithm~\ref{alg: PIE}. Therefore,
  if \([\vw_i \pm \lambda \vsigma] \subseteq \hatsC_s\), then
  \([\vw_i \pm \lambda \vsigma] \subseteq \hatsC_v\). Thus,
  \(\hatsC_s \subseteq \hatsC_v\) and
  \(p_s = \int_{\hat{\sC_s}} f_{\vw}(\vx) d\vw \leq \int_{\hat{\sC_v}}
  f_{\vw}(\vx) d\vw = p_p\). The proof for $p_s \leq p_g$ is similar.
\end{proof}

Intuitively, one may expect that \(p_p \leq p_g\); however, this does
not necessarily hold. For instance, if the gradient-based scaling
factor \(\rho\) is tool large, then \(\vw + j \vv^+\) may well lie
within the set of unsafe weights.

As presented in Section~\ref{sec: Evaluation}, both PIE and GIE allow
us to compute bounds for the probabilistic robustness of BNNs that are
significantly tighter than the SoA. Before presenting their
performance, we would like to emphasise the theoretical significance
of the iterative expansion, used in both approaches in
Example~\ref{ex: Bounded Support}.
\begin{example}
  \label{ex: Bounded Support}
  Consider a BNN with \(k\) nodes and posterior
  \(\sN(\vmu, \vsigma^2)\), the best lower bound that can be computed
  using a single orthotope is when the orthotope is centred at
  \(\vmu\). For orthotopes used in the pure sampling approach, this is
  an orthotope of the form \(\sO = [\vmu \pm \lambda \vsigma]\)
  (cf. Algorithm~\ref{alg: PIE}).
  For a fixed \(\lambda > 0\), the tightest lower bound that can be
  obtained for probabilistic robustness is \(p_{\lambda}^k\), where
  \(p_{\lambda} = Prob_{x \sim \sN(0, 1)}(|x| \leq \lambda)\). This
  means that the probabilistic robustness approximations computed
  using only pure sampling vanishes exponentially as the number of
  parameters of the network increases. This cannot be addressed by
  increasing the number of sampled orthotopes, as exponentially many
  orthotopes would be needed.

  To better understand this effect, we observe that for a BNN with
  only \(10^3\) parameters and \(\lambda = 3\), this gives us a lower
  bound of \(p_3^{1,000} \approx 0.9974^{1,000} < 0.08\).
\end{example}
%


%
\section{Evaluation}
\label{sec: Evaluation}
We evaluate the performance of our approach against the previously
discussed SoA approach for computing lower bounds for probabilistic
robustness~\citep{WickerLPK20}. As in previous studies, the approach
is evaluated on the MNIST dataset~\citep{Lecun+MNIST} and additionally
the CIFAR10 dataset~\citep{Krizhevsky09}. All experiments were run on
an AMD Ryzen Threadripper 3970X 32-core with 256GB RAM.

The evaluation is divided into three parts. \textbf{(1)} Benchmark
against the SoA pure sampling approach in
\citep{WickerLPK20}. \textbf{(2)} Ablation study on the use of
gradient-based dynamic scaling factors. \textbf{(3)} ablation study on
the number of expansion iterations.

For the MNIST baselines, we use 6 fully-connected BNNs
from~\citep{WickerLPK20}. However, we note that these networks have
narrow distributions (s.d. of $\approx 10^{-5}$). Therefore, we also
use two networks trained for broader posterior distributions (s.d. of
$\approx 10^{-3}$) that achieve comparable clean accuracy and show
greater natural resilience to adversarial attack, a key strength of
BNNs~\citep{Bortolussi+22} -- we provide the train details in
Appendix~\ref{sec: Training Details}. MNIST networks are identified
using an $\ell\text{x}n$ nomenclature with $\ell$, the number of
hidden layers, and $n$, the number of nodes per hidden layer.

We also include a Convolutional Neural Network (CNN) trained on the
CIFAR10 dataset. The CNN consists of 5 deterministic convolutional
layers followed by 3 fully-connected Bayesian layers with 256, 64, and
10 nodes, respectively. As in~\citep{WickerLPK20}, all parameter
posteriors are Gaussian.

\begin{figure}[t]
\centering
  \raisebox{4.1mm}{
    \includegraphics[width=.35\linewidth]{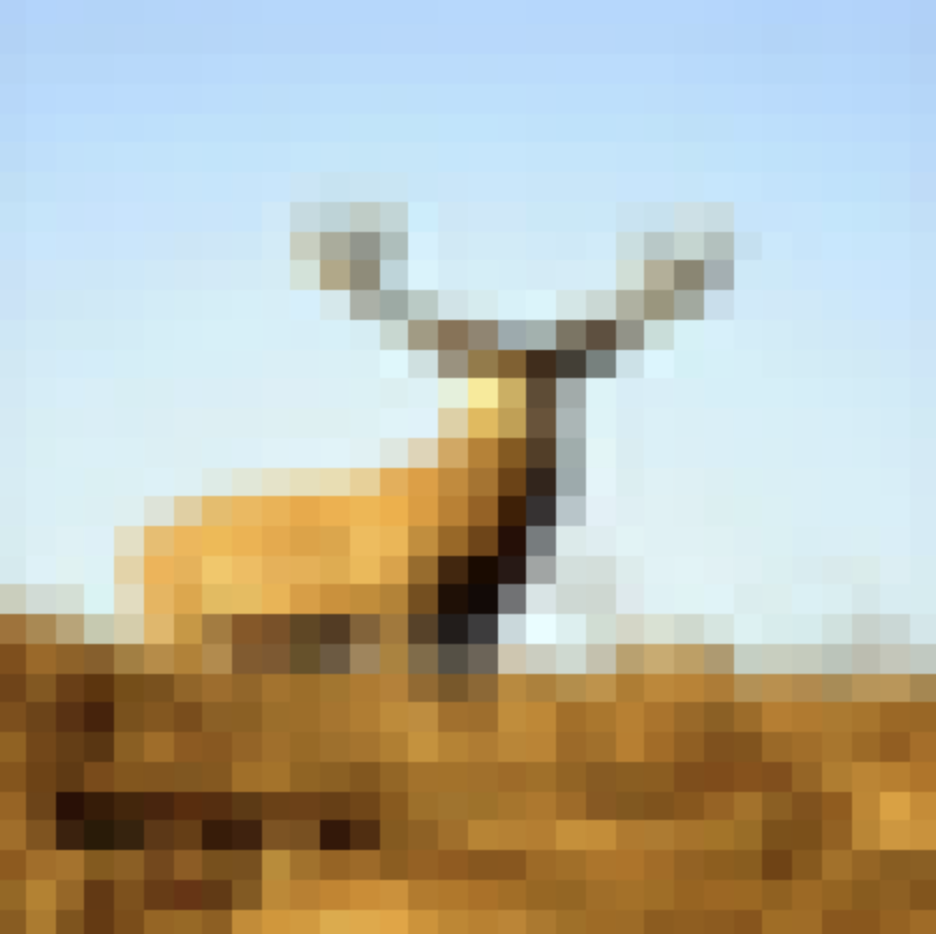}
  }
  \includegraphics[width=.609\linewidth]{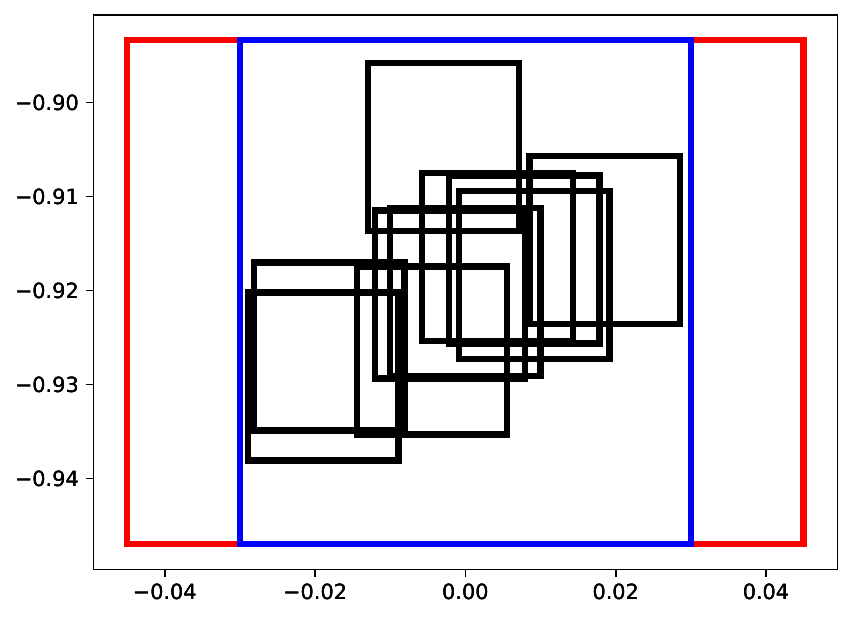}
  \caption{A 2-dimensional cross-section of the safe region \(\sC\)
    covered by pure sampling (black), PIE (blue), and GIE (Red) on the
    same image in CIFAR dataset.}
  \label{fig: CIFAR Covered Area}
\end{figure}
\subsection{Comparison of Pure Sampling, PIE, and GIE Approaches}
The performances of sampling-based methods are biased by the compute
available and the maximum number of iterations. To remove this bias,
we enforce an equal limit on the total number of LBP calls per
image. The results are presented in Table~\ref{tbl: Evaluation}. We
note that the GIE approach must also compute network gradients - a
step that neither PIE nor the pure sampling approach
require. Nevertheless, in practice, we find that the computation of
network gradients is much faster than the LBP used in all
approaches. We quantify and discuss the impact on runtime in
Appendix~\ref{sec: gie_timings}. Importantly, we noticed an error in
the official implementation in \citep{WickerLPK20} affecting the
computation of $p_{\text{safe}}$. We have corrected this error to
generate the results presented in Table~\ref{tbl: Evaluation} and
provide a detailed breakdown of the mistake in Appendix~\ref{sec:
  wicker_error}.

The results show a significant improvement over the pure sampling
method of \citep{WickerLPK20} in all cases with gradient-based scaling
providing the tightest guarantees. Moreover, we remark that in
obtaining these results the method from~\citep{WickerLPK20} depends on
careful selection of hyperparameters, namely the orthotope side
length, where the use of suboptimal values yields trivially small
guarantees -- for the results in Table~\ref{tbl: Evaluation} we used
grid-search to tune the hyperparameters in the pure sampling
approach. Differently, by using a small value for the initial
orthotope's side length, our approach is able to adapt to each image
and each network with little consideration for tuning hyperparameters.

In summary, the approach here presented considerably outperforms the
present SoA in estimating probabilistic robustness leading to bounds
that often more than double those in \citep{WickerLPK20}.

\begin{table*}[t]
  \begin{center}
    \resizebox{\textwidth}{!}{
    \begin{tabular}{ccccccccc}
\toprule
\multirow{2}{*}{Dataset} & \multirow{2}{*}{Network} & \multirow{2}{*}{\begin{tabular}[c]{@{}c@{}}Accuracy\\  (\%)\end{tabular}} & \multirow{2}{*}{\begin{tabular}[c]{@{}c@{}}Adversarial\\ Accuracy (\%)\end{tabular}} & \multirow{2}{*}{\begin{tabular}[c]{@{}c@{}}Input\\ Perturbation, \(\epsilon\)\end{tabular}} & \multicolumn{3}{c}{Probabilistic Certification (\%)}                                 \\ \cmidrule{6-8} 
                         &                          &                                                                           &                                                                                      &                                                                                             & \multicolumn{1}{c}{Sampling (\citeauthor{WickerLPK20})} & \multicolumn{1}{c}{Ours (PIE)} & Ours (GIE)    \\ \midrule
\multirow{8}{*}{MNIST}   & 1x128                    & 96                                                                        & 40                                                                                   & \multirow{3}{*}{0.025}                                                                      & \multicolumn{1}{c}{26.0}          & \multicolumn{1}{c}{56.9}       & \textbf{58.4} \\
                         & 1x256                    & 94                                                                        & 42                                                                                   &                                                                                             & \multicolumn{1}{c}{22.0}          & \multicolumn{1}{c}{50.0}       & \textbf{51.2} \\
                         & 1x512                    & 94                                                                        & 36                                                                                   &                                                                                             & \multicolumn{1}{c}{10.0}          & \multicolumn{1}{c}{35.0}       & \textbf{37.3} \\ \cmidrule{2-8} 
                         & 2x256                    & 92                                                                        & 16                                                                                   & \multirow{3}{*}{0.001}                                                                      & \multicolumn{1}{c}{72.0}          & \multicolumn{1}{c}{74.0}       & \textbf{74.0} \\
                         & 2x512                    & 82                                                                        & 2                                                                                    &                                                                                             & \multicolumn{1}{c}{22.0}          & \multicolumn{1}{c}{22.0}       & \textbf{26.1} \\
                         & 2x1024                   & 68                                                                        & 0                                                                                    &                                                                                             & \multicolumn{1}{c}{0.0}           & \multicolumn{1}{c}{0.0}        & 0.0           \\ \cmidrule{2-8} 
                         & 2x50                     & 96                                                                        & 70                                                                                   & 0.01                                                                                        & \multicolumn{1}{c}{58.6}          & \multicolumn{1}{c}{67.3}       & \textbf{71.8} \\ \cmidrule{2-8} 
                         & 2x150                    & 92                                                                        & 54                                                                                   & 0.001                                                                                       & \multicolumn{1}{c}{49.8}          & \multicolumn{1}{c}{54.9}       & \textbf{56.4} \\ \midrule
CIFAR10                  & CNN                      & 72                                                                        & 2                                                                                    & 0                                                                                           & \multicolumn{1}{c}{22.0}          & \multicolumn{1}{c}{27.1}       & \textbf{35.6} \\ \bottomrule
    \end{tabular}
  }
  \caption{Probabilistic robustness results for the pure sampling
    method from \citep{WickerLPK20} and ours. Adversarial accuracies
    were obtained using an \(\epsilon = 0.05\) bounded PGD attack for
    all networks. Probabilistic certification denotes the mean
    \(p_{\text{safe}}\) by each approach. Also, see the Appendix for a
    comparison of the probabilistic certification results obtained by
    each approach with the empirical probabilistic robustness values
    for each network.}
  \label{tbl: Evaluation}
  \end{center}
\end{table*}
\begin{figure}[t]
  \centering
  \includegraphics[width=.9\linewidth]{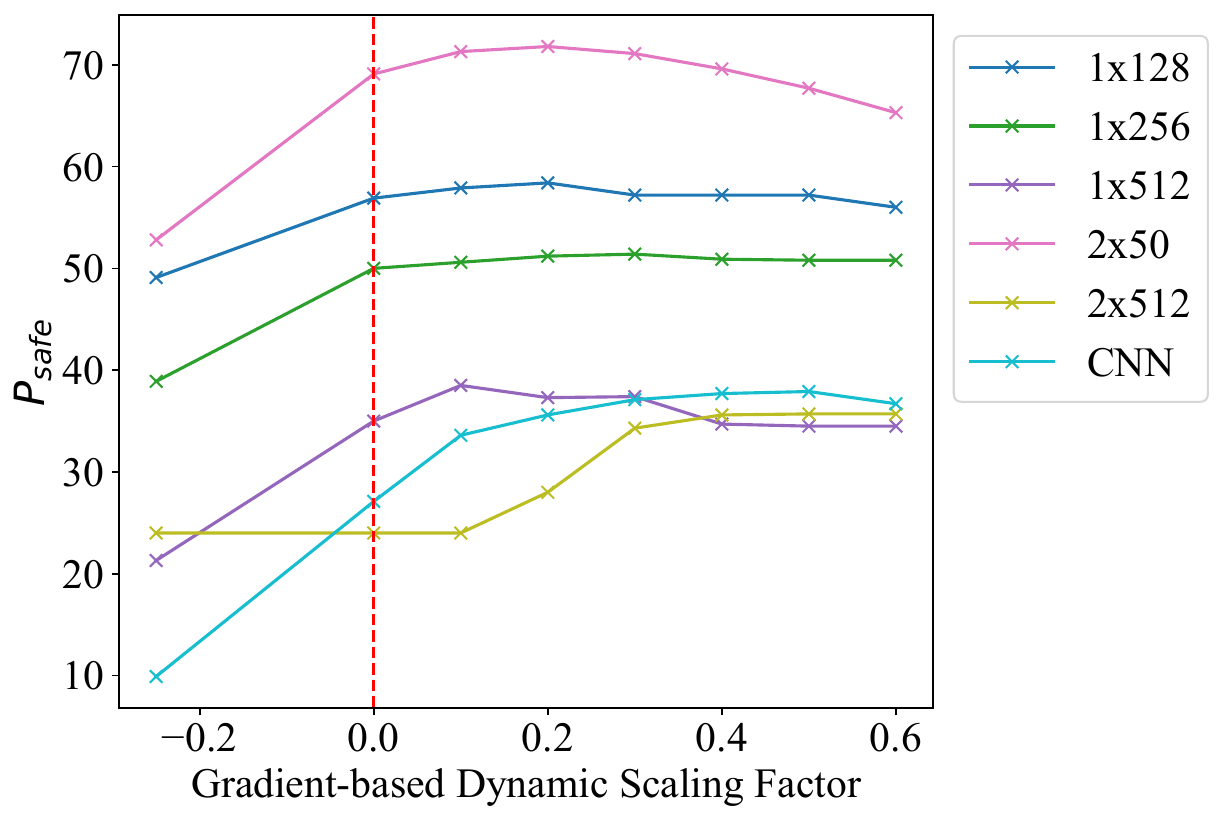}
  \caption{Probabilistic certification as a function of \(\rho\), the
    gradient-based scaling factor in Algorithm~\ref{alg: GIE}.}
  \label{fig: ablations}
\end{figure}
\subsection{ Ablation Study on the Use of Gradient-Based Dynamic
  Scaling Factors}
Here we explore the impact of using gradient-based dynamic scaling
factors.  In Figure~\ref{fig: ablations}, we plot the
\emph{probabilistic certification}, i.e., the mean
\(p_{\text{safe}}\), attained for a range of gradient-based scaling
factors and networks. Each point is averaged over 50 images from the
MNIST test set, same as \citep{WickerLPK20}.  In each case, we observe
that the optimal scaling factor is larger than 0 -- which represents
no gradient-guided scaling and is marked by the red dashed line. The
optimal gradient-based scaling factor varies between networks, as does
its effect on the overall $p_{\text{safe}}$. Examining the results for
individual images with a gradient-based scaling factor of 0 goes some
way to explaining this behaviour.

For the 1x256 network, we find that 89.1\% of the correctly classified
images have $p_{\text{safe}}$ values greater than 99.9\% or less than
0.1\%. On these networks, most of the images are either trivially
verifiable or prohibitively challenging, with only 10.9\% in
between. In such instances, the use of gradient-based scaling factors
has little effect, as $p_{\text{safe}}$ is either very close to 1, or
adversarial examples are found before $p_{\text{safe}}$ can reach
substantial values.

To further illustrate this point, we examined the CNN network, which
was most sensitive to the gradient-based scaling factor in the same
way. We found that only 68.6\% of the correctly classified images lay
outside the 0.1\% - 99.9\% boundary. To remove this dilution effect,
we examined how gradient-based dynamic scaling impacts a single image,
for which we use the same example given in Figure~\ref{fig: CIFAR
  Covered Area}. Using no gradient-based dynamic scaling on this
image, we achieve $p_{\text{safe}} = 36.5$ while using a
gradient-based scaling factor of 0.2 we reach
$p_{\text{safe}} = 97.8$, highlighting the importance of
gradient-based dynamic scaling on certain inputs.

\subsection{Ablation Study on the Number of Iterations}
\label{subsec: Iterations Ablation}
To better understand the inner workings of the iterative expansion
approaches presented in this paper, we calculated the number of
iterations it took Algorithm~\ref{alg: PIE} to reach the maximum
probabilistic certification presented in Table~\ref{tbl:
  Evaluation}. The results are available in Figure~\ref{fig:
  Iterations}. We observe that Algorithm~\ref{alg: PIE} takes 8
iterations to start computing a meaningful lower bound and at most 13
iterations to reach maximum probabilistic certification for the MNIST
networks.
The minimum number of iterations to obtain a meaningful lower bound
can be reduced, e.g., to a single iteration, by increasing the value
of \(\lambda\) in line 8 of Algorithm~\ref{alg: PIE}; however, this is
disadvantageous for two reasons: (1) one has to manually search for an
optimal value for \(\lambda\) similarly to the pure sampling approach,
(2) this can reduce the certified robustness obtained in later
iterations because the expansion rate of PIE increases with
\(\lambda\).
\begin{figure}[th]
  \centering
  \includegraphics[width=.778\linewidth]{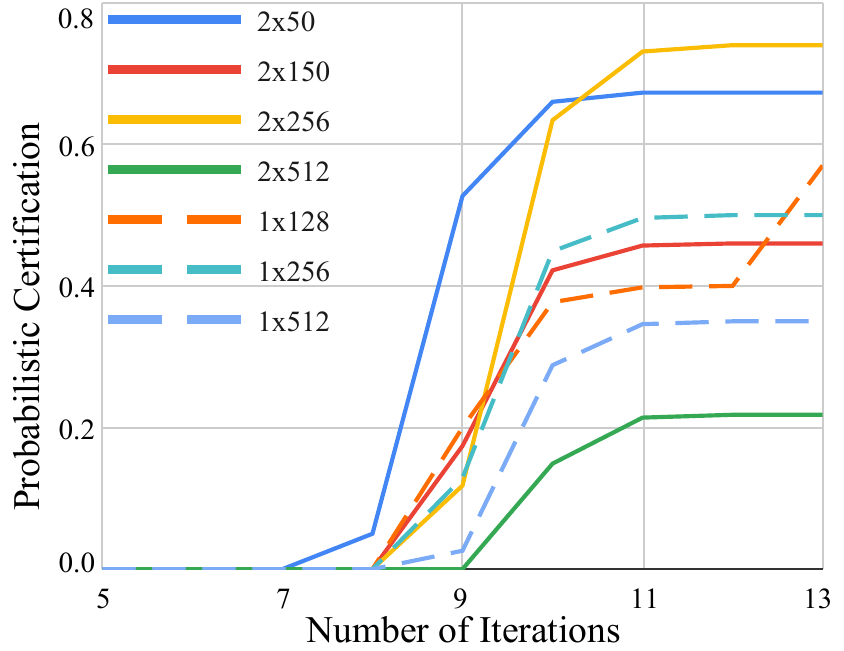}
  \vspace{-.3em}
  \caption{The number of expansion iterations before PIE reaches max
    probabilistic certifications in Table~\ref{tbl: Evaluation}.}
  \label{fig: Iterations}
\end{figure}
%


%
\section*{Conclusions}
As we discussed in the introduction, BNNs have been proposed as an
inherently robust architecture. This is particularly important in
safety-critical applications, where robust solutions are essential,
and models need to be shown to be highly robust before deployment.

Given this, it is essential to ascertain the robustness of BNNs in
production. Differently from plain neural networks, where robustness
verification can be recast into a deterministic optimisation problem,
establishing the robustness of BNNs is a probabilistic
problem. Finding an exact solution to this problem is infeasible even
for small models. As reported, proposals have been put forward to
approximate the probabilistic safety of BNNs by sampling regions in
the weight space; however, such approximations are loose and
underestimate the actual robustness of the model, thereby providing
relatively little information. This limits the possibility of using
BNNs in tasks where robustness needs to be verified before deployment.

We have put forward an alternative approach to search the weight
space, and therefore, approximate the lower bounds of probabilistic
safety in BNNs. The algorithms proposed provably produce tighter
bounds than the SoA sampling-based approaches. This conclusion was
validated in all experiments conducted. The approaches were evaluated
on all the previously proposed benchmarks, as well as new benchmarks
introduced in this paper, reporting a marked improvement on the bounds
of all benchmarks. In some cases, we obtain bounds that are 50\%
higher than the previously known bound of just over 20\%, opening the
way to better evaluate the robustness of BNNs before deployment.


%
\subsubsection*{Acknowledgements}
The authors are grateful to Peyman Hosseini for extensive comments on
previous versions of this paper and help with the experiments. This
work is supported by the UKRI Centre for Doctoral Training in Safe and
Trusted Artificial Intelligence, the “REXASI-PRO” H-EU project, call
HORIZON-CL4-2021-HUMAN-01-01, Grant agreement ID: 101070028, and the
UK Royal Academy of Engineering (CiET17/18-26).


%
\bibliography{bibliography}

\onecolumn
\appendix

\section{Limitations}
\label{sec: limitations}
Methods for verifying the probabilistic robustness of BNNs are
extremely sensitive to parameter count - this remains true for both
algorithms here proposed. Therefore scaling to larger networks is a
challenge due to two primary effects: Firstly, as network size
increases so does the parameter count and respective dimensionality of
the probability space. The higher this dimensionality, the more
important sampling \textit{large} orthotopes becomes, and the less
effective sampling \textit{multiple} orthotopes becomes. This reliance
on being able to verify single large orthotopes makes these algorithms
vulnerable to loose verification techniques. Secondly, as network size
increases it is well known that the tightness of bound propagation
techniques suffers, thus impacting the ability of our algorithms to
verify any orthotopes. Both of these limitations can be alleviated by
using tighter verification procedures in place of bound propagation.

%
\section{Correction to~\citep{WickerLPK20} implementation}
\label{sec: wicker_error}
We remarked in Section~\ref{sec: Evaluation} that while generating the
experimental results we identified an error in the implementation
in~\citep{WickerLPK20}. In the following, we present more details of
the issue identified. The results presented in the paper refer to the
corrected version of their implementation.

Given a series of $s$ potentially-overlapping orthotopes,
$\hatsC_1, \dots, \hatsC_s$, we can calculate the probability captured
in the series by computing the cumulative probability for each
orthotope $\hatsC_i$ and summing over the series, while taking into
account that we need to subtract the cumulative probabilities of the
intersections and so on according to the inclusion-exclusion
principle. Computing the cumulative probability of a single orthotope
is trivial as the distributions in each dimension are independent;
accordingly, we have
\begin{equation}
  \label{eq:correct_1}
  P(\hatsC_i) = \prod_{d=1}^{n_{\vw}}\frac{1}{2}\left(
    \erf(\frac{\vmu_d - \vell^i_d}{\vsigma_d \sqrt{2}}) -
    \erf(\frac{\vmu_d - \vu^i_d}{\vsigma_d \sqrt{2}})\right),
\end{equation}
where $\vmu_d$ and $\vsigma_d$ are the mean and standard deviation of
the $d$-th parameter and $\vell^i_d$ and $\vu^i_d$ are the lower and
upper bounds in the $d$-th dimension for the $i$-th orthotope. We can
then obtain the probability of the union of all $s$ orthotopes
according to the inclusion-exclusion principle. For example, if there
are no overlaps, we can simply sum $P(\hatsC_i)$'s for all $\hatsC_i$
and obtain
\begin{equation}
  \label{eq:correct_2}
  \hat{p}_{\safe} = \sum_{i=1}^s P(\hatsC_i).
\end{equation}
This is correctly presented in Algorithm~1 from~\citep{WickerLPK20}.

However, the code\footnote{At the time of writing, this is available
  at \texttt{https://github.com/matthewwicker/ProbabilisticSafetyforBNNs}.}
accompanying the paper implements this differently leading to large
overestimates (i.e., incorrect values) of the $p_{\safe}$ values as
presented in~\citep{WickerLPK20}.

The implementation first computes the dimension-wise probability,
\begin{equation}
  \label{eq:incorrect_1}
  P(\vw_d) = \sum_{i=1}^s \frac{1}{2}\left(
    \erf(\frac{\vmu_d - \vell^i_d}{\vsigma_d \sqrt{2}}) -
    \erf(\frac{\vmu_d - \vu^i_d}{\vsigma_d \sqrt{2}})\right),
\end{equation}
where $\vw_d$ is the $d$-th Bayesian parameter. The implementation
then computes the product of $P(\vw_d)$ across the $n_{\vw}$
dimensions,
\begin{equation}
  \label{eq:incorrect_2}
  \hat{p}_{\safe} = \prod_{d=1}^{n_{\vw}} P(\vw_d),
\end{equation}
where $\hat{p}_{\safe}$ is marked with a hat to represent an invalid
$p_{\safe}$ value. The error comes from the fact that the operations
in Equations~\eqref{eq:correct_1} and \eqref{eq:correct_2} are
noncommutative.

Below we show two code snippets taken from the~\citep{WickerLPK20}
implementation. The first shows the orthotopes being passed to the
`\texttt{compute\_interval\_prob\_weight}' function (lines 1-4). Note
that the parameters are separated into weights and biases and by layer
at this point. The second snippet shows the definition of the
`\texttt{compute\_interval\_prob\_weight}' function, where the
probability is calculated for each dimension over all orthotopes, as
in Equation~\eqref{eq:incorrect_1}. The returned parameter-wise
probabilities are then combined by multiplication as in
Equation~\eqref{eq:incorrect_2} producing $\hat{p}_{\safe}$, aliased
`\texttt{p}' in the code (lines 9-19).

We are confident this is an error in the implementation. The method
itself is correct, but the results produced by the implementation are
incorrect (overestimates). To validate this is indeed an error in the
code, we asked two expert ML authors to check this portion of the code
by making no suggestion it may have an error and both independently
identified the same error we report above.

\begin{listing*}[tb]
\begin{lstlisting}[language=Python]
pW_0 = compute_interval_probs_weight(np.asarray(vW_0), marg=w_margin, mean=mW_0, std=dW_0)
pb_0 = compute_interval_probs_bias(np.asarray(vb_0), marg=w_margin, mean=mb_0, std=db_0)
pW_1 = compute_interval_probs_weight(np.asarray(vW_1), marg=w_margin, mean=mW_1, std=dW_1)
pb_1 = compute_interval_probs_bias(np.asarray(vb_1), marg=w_margin, mean=mb_1, std=db_1)

# Now that we have all of the probabilities we just need to multiply them out to get
# the final lower bound on the probability of the condition holding.
# Work with these probabilities in log space
p = 0.0
for i in pW_0.flatten():
    p+=math.log(i)
for i in pb_0.flatten():
    p+=math.log(i)
for i in pW_1.flatten():
    p+=math.log(i)
for i in pb_1.flatten():
    p+=math.log(i)
#print math.exp(p)
return math.exp(p)
\end{lstlisting}
\end{listing*}
\begin{listing*}[tb]
\begin{lstlisting}[language=Python]
def compute_interval_probs_weight(vector_intervals, marg, mean, std):
    means = mean; stds = std
    prob_vec = np.zeros(vector_intervals[0].shape)
    for i in range(len(vector_intervals[0])):
        for j in range(len(vector_intervals[0][0])):
            intervals = []
            for num_found in range(len(vector_intervals)):
                interval = [vector_intervals[num_found][i][j]-(stds[i][j]*marg), vector_intervals[num_found][i][j]+(stds[i][j]*marg)]
                intervals.append(interval)
            p = compute_erf_prob(merge_intervals(intervals), means[i][j], stds[i][j])
            prob_vec[i][j] = p
    return np.asarray(prob_vec)
\end{lstlisting}
\end{listing*}
\begin{table}[t]
  \begin{center}
    \begin{tabular}{rc}
      \toprule
      Network & Relative Runtime ($\pm \%$) \\ \midrule
      1x128   & +3.3                        \\
      1x256   & +2.5                        \\
      1x512   & +1.4                        \\ \midrule
      2x256   & +0.8                        \\
      2x512   & +0.6                        \\
      2x1024  & +0.2                        \\ \midrule
      2x50    & +2.3                        \\
      2x150   & +0.4                        \\ \midrule
      CNN     & +0.5                        \\ \bottomrule    
    \end{tabular}%
    \caption{Impact of computing gradients for PIE algorithm on
      runtime. The data are presented as a percentage increase or
      decrease in runtime as compared to the \citep{WickerLPK20}
      approach on the same networks.}
    \label{tbl: runtime_results}
  \end{center}
\end{table}
\section{Evaluation of timing for PIE algorithm}
\label{sec: gie_timings}
In Table~\ref{tbl: runtime_results} we examine the impact on runtime
of generating the gradients required for the PIE algorithm. We observe
that in each case there is an increase in runtime, peaking at $3.3\%$
for the 1x128 network. The relative cost decreases as network size
increases as a result of the gradient computation scaling better than
the LBP computation with network size.

\section{Training Details}
\label{sec: Training Details}
In addition to evaluating our approaches on the SoA benchmarks, we
evaluated our approach on two fully connected BNNs and a convolutional
BNN (only the fully connected layers are Bayesian). Similarly to the
SoA, we used normal distributions for the prior (with mean 0 and
standard deviation 1) and the posterior of the Bayesian layers. Using
a wider standard deviation than the SoA, as expected, has resulted in
networks that are more adversarially robust (cf. Table~\ref{tbl:
  Evaluation}). We trained the networks using variation inference
\citep{Blundell+15} using a learning rate of 0.001 with Adam optimiser
\citep{KingmaB15}. All the hidden layers in the networks use ReLU
activations.

The fully connected networks are trained on the MNIST dataset and each
have 2 hidden layers with 50 and 150 nodes, respectively. The
convolutional network is trained on the CIFAR10 dataset and has 5
convolutional layers with filter size 3 and stride 2, followed by 3
fully connected Bayesian hidden layers of sizes x and y.

\section{Comparison against Empirical Probabilistic Robustness}
In Figure~\ref{fig: Comparison against Empirical}, we compare the pure
sampling approach of \citep{WickerLPK20}, as well as the PIE and GIE
approaches, presented in this paper, against the empirical values we
obtained for the probabilistic robustness of the networks in
Table~\ref{tbl: Evaluation}. We calculated the values for the
empirical probabilistic robustness by sampling from the BNN posterior
and using IBP to check whether classified correctly. Each value for
the empirical accuracy shows the percentage of samples that resulted
in correct classification.

We observe that PIE and GIE always perform favourably compared to pure
sampling, and in some instances, such as for 1-layer networks, PIE and
GIE significantly outperform pure sampling. Nevertheless, for most
networks, there is still a gap between the certified robustness
obtained by each of the approaches and the empirical probabilistic
robustness.
\begin{figure}[th]
  \centering
  \includegraphics[width=.85\linewidth]{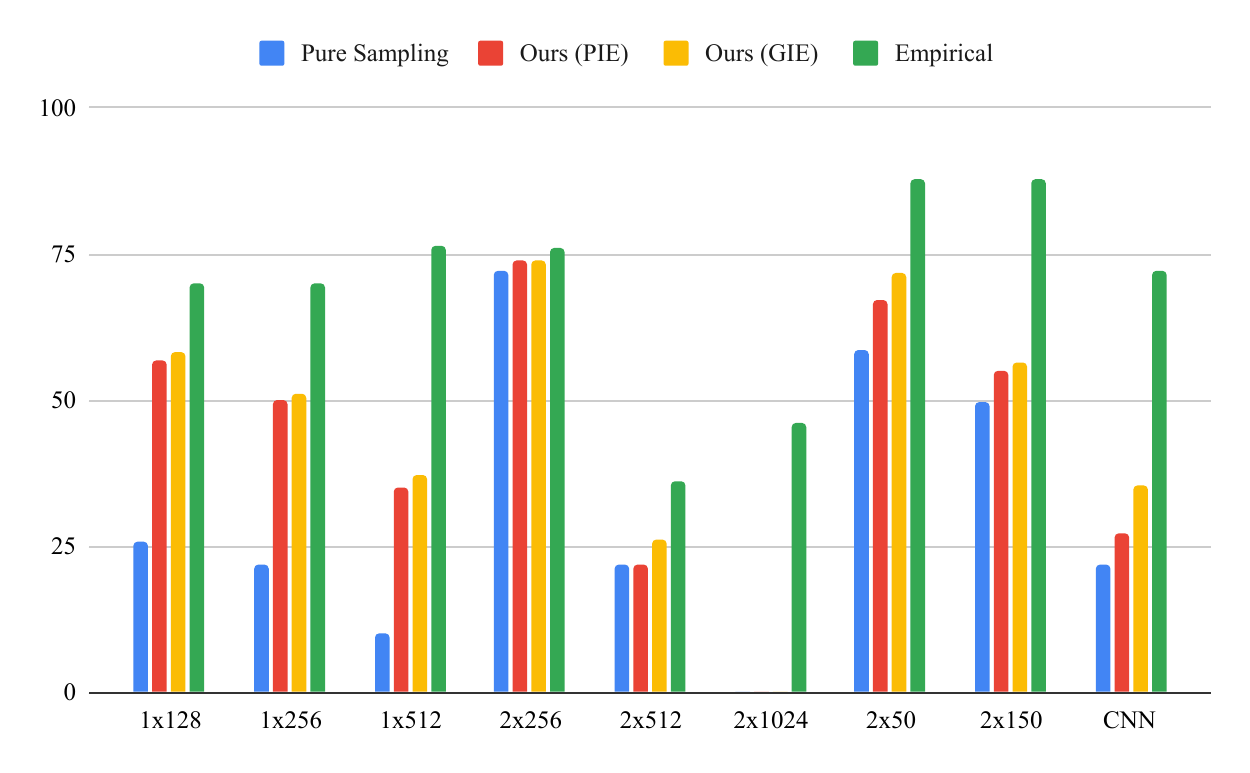}
  \caption{Comparison of the pure sampling, PIE, and GIE approaches
    against the empirical probabilistic robustness, obtained by
    sampling parameters from the BNN posterior and using IBP to verify
    whether they result in correct classification. 50 samples were
    used for calculating each empirical accuracy.}
  \label{fig: Comparison against Empirical}
\end{figure}
%


\end{document}